\documentclass{amsart}
\usepackage{amsthm}
\usepackage[dvips]{graphicx}

\newtheorem{lemma}{Lemma}

\newcommand{\rg}[1]{\mbox{\bf #1}}

\newcommand{\id}[1]{\mathcal{#1}}

\newcommand{\rf}[1]{(\ref{#1})}

\newcommand{\F}{\mathbb{F}}

\begin{document}
\title[Attacking Fuzzy Vault]{The Fuzzy Vault for Fingerprints is Vulnerable
  to Brute Force Attack} 
\author{Preda Mih\u{a}ilescu}
\address[P. Mih\u{a}ilescu]{University of G\"ottingen}
\email[P. Mih\u{a}ilescu]{preda@uni-math.gwdg.de}
\thanks{The research was completed while the author is holding a
research chair sponsored by the Volkswagen Stiftung} 
\date{Version 1.0, \today}



\vspace{2.0cm}
\begin{abstract}
  The \textit{fuzzy vault} approach is one of the best studied and well
  accepted ideas for binding cryptographic security into biometric
  authentication. We present in this paper a brute force attack which improves
  upon the one described by Clancy et. al. \cite{CKL} in an implementation of
  the vault for fingerprints. On base of this attack, we show that three
  implementations of the fingerprint vault are vulnerable and show that the
  vulnerability cannot be avoided by mere parameter selection in the actual
  frame of the procedure.

We also give several suggestions which can improve the fingerprint
vault to a cryptographically secure algorithm. In particular, we
introduce the idea of \textit{fuzzy vault with quiz} which draws upon
information resources unused by the current version of the
vault. This is work in progress, bringing important security
improvements and which can be adapted to the other biometric
applications of the vault.
\end{abstract}
\maketitle
\section{Introduction}
Secure communication relays on trustable authentication. The most wide
spread authentication methods still use passwords and pass-phrases as
a first step towards identity proving. Secure pass-phrases are hard to
remember, and the modern user needs a large amount of dynamic
passwords for her security. This limitation has been known for a long
time and it can in part be compensated by the use of chip cards as
universal access tokens.

Biometrical identification, on the other hand, is based on the physical
identity of a person, rather than on the control of a token. Reliable
biometric authentication would thus put an end to password insecurity, various
repudiation disputes and many more short comings of phrase or token based
identities. Unlike the deterministic keys which are common for cryptography,
the biometric data are only reproducible within vaguely controlled error
bounds, they are prone to various physical distortions and have quite low
entropy.

Overcoming the disadvantages of the two worlds by using their reciprocal
advantages is an important concern. Unsurprisingly, we look back to almost a
decade in which the biometrics community developed an increasing interest in
the security and privacy of biometrical systems. We refer to the survey
\cite{UPPJ} of Uludag et. al. on biometric cryptosystems for an update
overview of the work in this area.

Researchers from cryptography and coding theory, attempted to develop new
concepts allowing to model and evaluate biometric data from an information
theoretical point of view. The resulting algorithms deal with the specific
restraints of biometrics: non - uniformly distributed data with incomplete
reproducibility and low, hard to estimate, entropy. Both communities are
motivated by the wish to handle biometrics like a classical password, thus
protecting it by some variant of one-time functions and performing the
verification in the image space. Unlike passwords, the biometrics are not
deterministic. This generates substantial challenges for the verification
after one - time function transforms. Furthermore, disclosure of biometric
templates is considered to be a more vital loss then the loss of a password.

Juels and Wattenberg \cite{JW} and then Juels and Sudan \cite{JS} have
developed with the \textit{fuzzy commitments } and \textit{fuzzy vault} two
related approaches with a strong impact in biometric security. The papers of
Dodis et. al.  \cite{DORS, BDKOS} can be consulted for further theoretic
development of the concepts of Juels et. al.  and their formalization in an
information theoretic frame. It is inherent to the problem, that core concepts
of the theory, such as the entropy of a biometric template, are hard to even
estimate. Thus the security proofs provided by the theory do not translate
directly in practical estimates or indications. As a consequence, we shall see
that concrete implementation studies of the fuzzy vault have to be considered
for security investigations, whilst the original paper \cite{JS} does not
contain sufficient directions of application in order to allow such an
analysis. Since we focus on the application of the vault to the biometry of
fingerprints, we shall often use the term of \textit{fingerprint vault} for
this case.

Clancy, Kiyavash and Lin gave in 2003 \cite{CKL} a statistically supported
analysis for a realistic implementation of the vault to fingerprints. The
authors observe from the start that the possible parameter choices in this
context are quite narrow for allowing sufficient security; they succeed to
define a set of parameters which they claim to provide the cryptographically
acceptable security of $O(2^{69})$ operations for an attack. We show that
faster attacks are possible in the given frame, thus making brute force
possible. The analysis in \cite{CKL} are outstanding and have been used
directly or indirectly in subsequent papers; the good security was obtained at
the price of a quite high error probability ($20\% - 30\%$).

Uludag and Jain provided in \cite{UJ1, UJ2} an implementation of the fuzzy
vault for fingerprints which uses alignment help - data and was applied to the
fingerprints from the database \cite{FVC}. This improves the identification
rate; however some of the simplifications they make with respect to \cite{CKL}
reduce security quite dramatically. The ideas of these authors could however
very well be combined with the more defensive security approach of Clancy et.
al.

Yang and Verbauwhede \cite{YV} describe an implementation of the vault, with
no alignment help, which follows closely the concepts of \cite{CKL} and
focuses upon adapting to various template qualities and the number of
minutiae recognized in these templates.

We describe the original fuzzy vault in Section 2 and argue that the security
proofs and remarks given in \cite{JS} are insufficient for fingerprint
applications. In the same section we describe and prove our approach to brute
force attack. In Section 3 we discuss the various implementations mentioned
above and show that brute force can be performed in feasible time in all
instances. In Section 4 we discuss possible variants and alternatives and
suggest the use of additional sources of information, thus raising the
security vault to acceptable cryptographic standards.
\section{The Fuzzy Vault}
The \textit{fuzzy vault} is an algorithm for hiding a secret string $S$ in
such a way that a user who is in possession of some additional information $T$
can easily recover $S$, while an intruder should face computationally
infeasible problems in order to achieve this goal. The information $T$ can be
fuzzy, in the sense that the secret $S$ is locked by some related, but not
identical data $T'$. Juels and Sudan define the vault in general terms,
allowing multiple applications. Biometry is one of them and we shall restrict
our description directly to the setting of fingerprints. Generalizations are
obvious, or can be found in \cite{JS, DORS, BDKOS}.

The string is prepared for the transmission in the vault as follows. Let $S
\in \{0, 1\}^{*}$ be a secret string of $l$ bits length. The user (Alice, say)
that wishes to be identified by the string $S$ has her finger scanned and a
\textit{locking set} $\id{L}$ comprising the Cartesian coordinates of $t$
minutiae in the finger scan is selected from this finger template $T'$; the
couples of coordinates are concatenated to single numbers $X_i = (x_i || y_i)
\in \id{L}$. One selects a finite field $\F_q$ attached to the vault and lets
$k' + 1 = \lceil \frac{l}{\log_2(q)}\rceil$ be the number of elements in
$\F_q$ necessary to encode $S$. One assumes that $0 < \max_{X \in \id{L}} X <
q$ and maps $X \hookrightarrow \F_q$ by some convention. Selecting $f(X) \in
\F_q[ X ]$ to be a polynomial of degree $ k - 1 > k'$ with coefficients
which encode $S$ in some predetermined way, one builds the \textit{genuine}
set:
\[ \id{G} = \id{G}(\F_q, S, t, k, \id{L}) = \left\{( X_i, Y_i) \ : \ X_i \in 
\id{L}; Y_i = f(X_i) \right\} , \] which encodes the information of $S$. The
genuine verifier Bob has an original template $T$ of Alice's finger and should
use this information in order to recover $f(X)$ and then $S$. In order to make
an intruder's (Victor, say) attempt to recover $S$ computationally hard, the
genuine set is mixed with a large set of \textit{chaff} points
\[ \id{C} = \{ \ (U_j, W_j) \ : \ j = 1, 2, \ldots, r - t \} , \]
with $U_j \not \in \id{L}$ and $W_j \neq f(U_j)$; the chaff points should be
random uniformly distributed. Chaff points and genuine lists are shuffled to a
common vault with parameters:
\[ \id{V} = \id{V}(k, t, r, \F_q) = \id{G} \cup \id{C} .\]

Upon reception, Bob will generate an \textit{unlocking} set $\id{U}$. This set
contains those $X_i$ coordinates of vault points, which well approximate
coordinates of minutiae in $T$. Note that in order to have a reasonable
approximation of minutiae coordinates of the same finger in different
templates, this templates must 
\begin{itemize}
\item[a)] Have negligeable non linear distortions. 
\item[b)] Be aligned modulo affine transforms.
\end{itemize}
The second condition is addressed in \cite{UJ2}. The unlocking set may be
erronated, thus either allowing some chaff points which are closer to $T'$
then locking points, or making the choice of a sufficiently large unlocking
set hard. Both problems can be addressed within given limits by error
correcting codes. Thus Juels and Sudan suggest using Reed Solomon codes for
decoding $f(X)$.

The security argumentation in \cite{JS} is based upon the expectation that the
chaff points will build an important amount of subsets of $t$ elements, whose
coordinates are interpolated by polynomials $g(X) \in \F_q[ X ]$ of degree
$k$, thus hiding the correct value of $f(X)$ from Victor among the set of
polynomials $g(X)$. The argument is backed up by the following Lemma, a proof
of which can be found in \cite{JS, CKL}. 
\begin{lemma}
\label{l1}
For every $\mu, \ 0 < \mu < 1$ and every vault $\id{V}(k, t, r, \F_q)$, there
are at least $\frac{\mu}{3} \cdot q^{k-t} \cdot (r/t)^t$ random polynomials
$g(X) \in \F_q[ X ]$ such that $\id{V}$ contains $t$ couples $(U_j, g(U_j))$.
\end{lemma}

\subsection{A brute force attack}
If Victor intercepts a vault $\id{V} = \id{V}(k, t, r, \F_q)$, but has no
additional information about the location of minutiae or some of their
statistics, he may still try to recover $S$ by brute force trials. For this he
needs to find $k$ points in the genuine list $\id{G}$. The chances that $k$
points of the vault are also in the genuine list are:

\begin{eqnarray}
\label{prob}
 1/ \rg{P} = \frac{\binom{r}{k}}{\binom{t}{k}} \sim (r/t)^k < 1.1 \cdot
 (r/t)^k, \quad \hbox{for } \quad r > t > 5.
\end{eqnarray}
This, together with the fact that the probability that a random pair $(X, Y)
\in \F^2_q$ lays on the graph of a given polynomial $f(X) \in \F_q[ X ]$ is
equal to $P[ Y = f(X) ] = 1/q$, yield the ground for the proof of Lemma
\ref{l1}. 

Lagrange interpolation of a polynomial of degree $k$ can be done in $O(k
\log^2(k))$ operations \cite{GG}; checking whether an additional point $(U,W)$
lays on the graph of $f(X)$ (so $W = f(U)$) requires $O(k)$ steps, so $K =
O(\log^2(k))$ such verifications can be done at the cost of one interpolation.

We assume now with Clancy et. al., that there is a degree $k -1 < D < t$ which
is minimal with the property that among all polynomials $g(X) \in \F_q[ X ]$
of degree $k-1$ which interpolate vault points, $f(X)$ is the only one which
interpolates with probability close to $1$ at least $D$ points. This yields a
criterium for identifying $f(X)$ \cite{CKL}:
\begin{lemma}
The complexity of the brute force attack problem using a suitable value $D$ as
above is $C_{bf} = \binom{r}{D} / \binom{t}{D}$.
\end{lemma}
We suggest the following brute force attack, which is not dependent on the
value of $D$ and is stronger then the previous:
\begin{lemma}
\label{l2}
Let $\id{V}(k,t,r,\F_q)$ be a fuzzy fingerprint vault and $k < t$ be chosen as
above. Then an intruder having intercepted $\id{V}$ can recover the secret $S$
in $R = C \cdot (r/t)^k$ operations, where $C < 8 r k$.
\end{lemma}
\begin{proof}
We have shown that in less then $< 1.1 \cdot (r/t)^k$ trials, Victor can find
a set of $k$ points from the locking set $\id{L}$. In order to find such a set
and then $S$, for each $k$ - tuple $\id{T} = (X_i, Y_i)_{i=1}^k \subset
\id{V}$ Victor has to
\begin{itemize}
\item[1.] Compute the interpolating polynomial $g_{\id{T}}(X)$. It is proved
  in \cite{GG} that the implicit constant for Lagrange interpolation is
  $6.5$; let $K = 6.5 \cdot \log^2(k)$. Thus computing all the interpolation
  polynomials requires $< 7.2 \cdot k \log^2(k) \cdot (r/t)^k$ operations.
\item[2.] Search a point 
\begin{eqnarray}
\label{cond}
(U, W) \in \id{V} \setminus \id{T} \quad \hbox{ such that } \quad g(U) = W. 
\end{eqnarray}
  This requires the equivalent of $r/K$ Lagrange interpolations. If no point
  is found, then discard $\id{T}$.
\item[3.] If $\id{T}$ was not discarded, search for a further point which
  verifies \rf{cond}. This step is met with probability $1/q$. If a point is
  found, add it to $\id{T}$; otherwise discard $\id{T}$. 
\item[4.] Proceed until a break condition is encountered (no more points on
  the graph of $g(X)$) or $D$ points have been found in $\id{T}$, and thus
  $g(X) = f(X)$ with high probability.
\end{itemize}
Adding up the numbers of operations required by the steps 1.-4., with weights
given by the probabilities of occurrence, one finds:
\[ R < 7.2 \cdot (r/t)^k \cdot k \cdot K \cdot r \cdot \sum_{j=0}^{\infty} (1/q)^ j = 7.2 \cdot (r/t)^k \cdot k \cdot K \frac{r q}{K
  (q-1)} < 8.0 \cdot (r k) \cdot (r/t)^k, \] as claimed. 
\end{proof}

Here are some remarks on factors that influence the complexity of the brute
force attack:
\begin{itemize}
\item[(i)] \textit{Restricting the region of interest} from which Victor
  chooses points for his unlocking set is irrelevant, if minutiae are assumed
  to be uniformly distributed over the template. In this case, $r$ and $t$ are
  scaled by the same factor and thus $r/t$ and the complexity of brute force
  remains unchanged.
\item[(ii)] The complexity grows when \textit{increasing the degree $k$ of the
    polynomial $f(X)$}. However, high degrees $k$ require large unlocking
  sets, which may is a problem for average quality fingerprints and scanners.
  Thus one can only augment the degree to an extent which depends on the
  quality of both scanner and fingerprint. The issue of adapting to these
  factors is addressed in \cite{YV}.
\item[(iii)] The complexity grows when \textit{increasing the number of chaff
    points}. There is a bound to this number, given by the size of the image
  on the one side and the variance in the minutiae location between various
  data capturings and extractions \cite{CKL}, on the other. Clancy and his
  coauthors find empirically the lower bound $d \geq 10$ for the distance
  between chaff points, and this distance was essentially respected also by
  the subsequent works.
\item[(iv)] The complexity grows when \textit{reducing the size $t$ of the
    genuine list}. This is however also detrimental for unlocking, since it
  may reduce the size of the unlocking set below the required minimum.
\end{itemize}

What can be inferred about the security of fingerprint vaults from the seminal
paper \cite{JS}? First, one observes that Juels and Sudan suggest the use of
error correcting codes, thus avoiding to transmit in the vault explicite
indications to whether an interpolation polynomial is the correct $f(X)$.
Uludag and Jain suggest on the other hand in \cite{UJ2} the use of CRC codes:
thus $S$ is padded by a CRC code, adding $1$ to the minimal degree $k'$ needed
to encode $S$. Upon decoding, Bob can check the CRC and ascertain that he
found the correct secret. This simplifies the unlocking procedure, but also
allows the attacker to verify if he has found a correct unlocking set. Does
this bring advantages to Victor? The odds to find a correct CRC are equal to
the probability that $k+1$ points are interpolated by the same polynomial of
degree $k-1$. Thus, if the degree $k > k'$, Victor has no gain, as follows
form the Lemma \ref{l2}.

It is shown in the Chapters 4 and 5 of \cite{JS} that the amount of chaff
points is essential for security. The suggested minimum lays at about $r \sim
10^4$. For fingerprints, a large amount of chaff points naturally decreases
the average distance between these points in the list. The value $r = 10^4$
leads to an average distance of $2-5$ pixels between the point coordinates,
depending on the resolution of the original image. This is below realistic
limits as mentioned in (iii) above. At this distance, even in presence of a
perfect alignment, the genuine verifier Bob should need some additional
information - like CRC or other - providing the confirmation of the correct
secret. But such a confirmation is contrary to the security lines on base of
which Juels and Sudan make there evaluation. 

There is an apparent conflict between the general security proofs in \cite{JS}
and realistic applications of the fuzzy vault to fingerprints. In the
implementation chapter, the authors explicitely warn that \textit{applications
  involving privacy-protected matching} cannot achieve sufficient security. It
is conceivable that fingerprint matching would be considered by the authors as
belonging to this category.

\section{Implementations of the fingerprint vault}
We start with the most in depth analysis of security parameters for the
fingerprint vault, which was done by Clancy and coauthors in \cite{CKL}. The
paper focuses on applications to key release on smart cards. They suggest
using multiple scans in order to obtain, by correlation, more reliable locking
sets. As mentioned above, the \textit{variance of minutiae locations} which
they observed in the process leads to defining a minimum distance between
chaff (and genuine) points, which is necessary for correct unlocking. This
minimal distance $d \sim 11$ implies an upper bound for the size $r$ of the
vault and thus the number of chaff points!

The authors use very interesting arguments on packing densities and argue that
in order to preserve the randomness of chaff points, these cannot have maximal
packing density. On the other hand, assuming that the intruder has access to a
sequence of vaults associated to the same fingerprint and he can align the
data of these vaults, the randomness of chaff points allows a correlation
attack for finding the genuine minutiae. 

This observation suggests rather using perfectly regular high density chaff
point packing. These are hexagonal grids with mutual distance $d$ between the
points. The genuine minutiae can be rounded to grid points, and Victor will
have no clue for distinguishing these from the chaff points. We shall comment
below on this topic.

The implementation documented in this paper suggests the following parameters
for optimal security: $k = 14, D = 17, t = 38, r = 313$. The brute force
attack in Lemma \ref{l2} is more efficient then the one of Theorem 1 of
\cite{CKL}, on which they base their security estimates. Using the above
parameters and Lemma \ref{l2}, we find an attack complexity of $\sim 2^{50}$.
Comparing this to the complexity of genuine unlocking yields a security factor
$F \sim 2^{44}$, which is below cryptographic security, unlike the $2^{69}$
deduced by Clancy et. al. in \cite{CKL}. Since the empirical values of $t$ in
\cite{CKL} range in the interval $[20, 60]$ with expected value $t = 38$, it
is possible that their estimate was gained by using the minimal value of $t$
which corresponds to maximal security. However, this leads to a small
difference between $t$ and $k$ and this may reduce the rate of correct
decodings.

By the balanced arguments used in the parameter choice, the security bounds
obtained on base of \cite{CKL} are an indication of the vulnerability of the
fingerprint vault in general.

Yang and Verbauwhede describe in \cite{YV} an implementation of the vault, in
which the degree of the polynomials $f(X)$ varies in dependence of the size
$t$ of the genuine list, which itself depends directly on template quality.
From the point of view of security, the paper can be considered as a follow up
of \cite{CKL}, which addresses the problem of poor image quality with its
consequences for the size of the locking set. The size of the secret $S$ and
polynomial degrees are adapted to the size of locking sets. The proposal is
consistent, its vulnerability to attacks is comparable to \cite{CKL} in
general, and higher, when adapting to poor image quality.

The major contributions of Uludag and Jain in \cite{UJ2} is to provide a
useful set of \textit{helper data} for easing image alignment. This has an
important impact on the identification rate. As mentioned above, they bring
the elegant and simple proposal of adding a CRC to the secret, thus easing the
unlocking work. We discussed above the issue of the security risk increasment:
this is arguably small. On the other hand, the degree of $k = 8$ for the
polynomial $f(X)$ and vault size $r = 200$, whilst $t = 25$, makes their
system more vulnerable, with an absolute attack complexity of $\sim 2^{36}$.
Better security can be achieved in this system by using the parameters of
\cite{CKL}.

\section{Security discussion}
We discuss in this section several variants for improving security of the
fingerprint vault. Future research shall analyze the practicality of some of
them.
\subsection{Using more fingers}
We have shown that the parameters $r, t, k$, allowing to control the security
factor, are naturally bounded by image size, variance of minutiae location and
average number of reliable minutiae. They cannot thus be modified beyond
certain bounds and it is likely that this bounds have been well established in
\cite{CKL}. It lays thus at hand to propose using for instance the imprints of
two fingers rather then only one, for creating the vault. This leads
practically to a squaring of the security factor.

\subsection{Non - random chaff points}
As mentioned above, it is suggested in \cite{CKL} that chaff points should
have random distribution; this leads to halving the packing density compared
to maximal density packing. However, one can embrace the opposite attitude.
This consists in laying a hexagonal grid of size $d = 11$, proposed by the
authors upon the fingerprint template, thus achieving maximal packing. Each
grid point will be attached to some vault point - chaff or genuine. Thus
Victor will have no means for distinguishing between chaff points and genuine
ones, despite of the regularity of the grid. 

Thanks to the error correcting codes, the genuine points can always be
displaced by a distance at most $d/2$ to a grid point. This strategy improves
the security of the vault in two ways: by doubling the size $r$ of the vault
and by avoiding correlation attacks. The consequences need still be analyzed.

\subsection{Quizzes using additional minutiae information}
There is more information in a minutia than its mere coordinates. Such are for
instance its orientation, the lengths and curvatures of incoming lines,
neighboring data, etc. We propose to attach to each minutia a \textit{quiz}
which can be solved in robust manner by Bob, but which introduces for Victor
several (say $b$) bits of uncertainty per minutia. Thus for polynomial degree
$k$, the security may be increased by a factor of $2^{k b}$.

Here is a simple example of how a quiz functions for the case of the
orientation of minutiae. Let $X$ be the concatenated coordinates of a fixed
minutia and let $\alpha$ be its orientation, in a granularity of $\pi/n$, for
some small integer $n$. Then, along with $(X, f(X))$, the vault will also
contain a value $\beta = \alpha + j \pi/n$: thus the minutia is represented by
$(X, Y, \beta)$. Upon reception, Bob computes the integer $0 \leq j < n$ such
that $j \pi = n \alpha - n \beta \bmod \pi$.  The value of $j$ will then
encode a certain transformation $Y' = T(Y)$ of the received value $Y$ and the
interpolating value will be set to be $Y' = f(X)$. Note that the vault creator
has control on the generation of $\beta$ and it may be chosen such that $j$
can be safely recovered by the genuine user. For chaff points, $\beta$ is
random. Several robust additional informations may as well increase the
security of the fingerprint vault to a cryptographically acceptable level.

\subsection{The alternative of cryptographic security}
These observations lead to the question: is the use of one - way functions and
template hiding an intrinsic security constraint, or just one in many
conceivable approaches to securing biometric authentication? The second is the
case, and it is perfectly feasible to construct a secure biometric
authentication system based on the mechanisms used by state of the art
certification authorities. The mechanisms are standard and have been
implemented by some providers. An important advantage of public key
cryptography is that it allows attaching time stamps to transmitted finger
templates, thus reducing the consequences in the event that a template is
compromised.

\section{Conclusions} 
It has been attempted to achieve security in biometric applications either by
using one-way functions adapted to the specificities of biometric data, or by
direct application of strong cryptographic techniques. We showed that one of
the leading methods of the first category, the fuzzy vault, allows a simple
attack to its instantiation for fingerprint data \cite{CKL, UJ1, UJ2, YV}. We
have brought some suggestions which may help raising the security level of the
fingerprint vault to cryptographic acceptable values.

One may argue that similar attacks could be possible to other related methods
and thus cryptographic security is preferable, whenever it can be achieved or
afforded. Subsequent work should consider variants of the one - way function
ideas which could meet the standards of cryptographic security. More in depth
statistical studies concerning the amount of information available from
various fingerprint related data are called for, thus providing a solid
foundation for security claims.

Also, cryptographic security can be brought in a wide scale of variants;
analyzing pros and contras of such variants is an open topic.

\textbf{Acknowledgment} I thank K. Mieloch and U. Uludag for enlighting
discussions and remarks.

\end{document}